\documentclass[conference]{IEEEtran}
\IEEEoverridecommandlockouts
\usepackage{cite}
\usepackage{amsmath,amssymb,amsfonts}
\usepackage{algpseudocode}
\usepackage{algorithm}
\usepackage{graphicx}
\usepackage{textcomp}
\usepackage{xcolor}
\usepackage{nicefrac}
\usepackage[frozencache,cachedir=.]{minted}
\usepackage{pgfplots}
\usepackage[english]{babel}
\usepackage{amsthm}
\usepackage{booktabs}

\newtheorem{theorem}{Theorem}

\newtheorem{lemma}[theorem]{Lemma}

\newcommand{\ceil}[1]{\lceil #1 \rceil}

\graphicspath{{./figures/}}

\def\BibTeX{{\rm B\kern-.05em{\sc i\kern-.025em b}\kern-.08em
    T\kern-.1667em\lower.7ex\hbox{E}\kern-.125emX}}
\begin{document}


\title{Parallel Actors and Learners: A Framework for Generating Scalable RL Implementations*\\
\thanks{This work has been sponsored by the U.S. Army Research Office (ARO) under award number W911NF1910362 and the U.S. National Science Foundation (NSF) under award numbers 2009057.}
}

\author{\IEEEauthorblockN{Chi Zhang}
\IEEEauthorblockA{\textit{Department of Computer Science} \\
\textit{University of Southern California}\\
Los Angeles, USA \\
zhan527@usc.edu}
\and
\IEEEauthorblockN{Sanmukh Rao Kuppannagari, Viktor K Prasanna}
\IEEEauthorblockA{\textit{Department of Electrical and Computer Engineering} \\
\textit{University of Southern California}\\
Los Angeles, USA \\
kuppanna@usc.edu, prasanna@usc.edu}
}

\maketitle

\begin{abstract}
Reinforcement Learning (RL) has achieved significant success in application domains such as robotics, games and health care. However, training RL agents is very time consuming. Current implementations exhibit poor performance due to challenges such as irregular memory accesses and thread-level synchronization overheads on CPU.
In this work, we propose a framework for generating scalable reinforcement learning implementations on multi-core systems. Replay Buffer is a key component of RL algorithms which facilitates storage of samples obtained from environmental interactions and data sampling for the learning process. We define a new data structure for Prioritized Replay Buffer based on $K$-ary sum tree that supports asynchronous parallel insertions, sampling, and priority updates. To address the challenge of irregular memory accesses, we propose a novel data layout to store the nodes of the sum tree that reduces the number of cache misses. Additionally, we propose \textit{lazy writing} mechanism to reduce thread-level synchronization overheads of the Replay Buffer operations. Our framework employs parallel actors to concurrently collect data via environmental interactions, and parallel learners to perform stochastic gradient descent using the collected data. Our framework supports a wide range of reinforcement learning algorithms including DQN, DDPG, etc. We demonstrate the effectiveness of our framework in accelerating RL algorithms by performing experiments on CPU + GPU platform using OpenAI benchmarks. 
Our results show that the performance of our $K$-ary sum tree based Prioritized Replay Buffer improves the baseline implementations by around 4x$\sim$100x. Our proposed synchronization optimizations improve the performance by around 2x$\sim$4.4x compared with using a global lock. 
By plugging our Replay Buffer implementation into existing open source reinforcement learning frameworks, we achieve 1.19x$\sim$1.75x speedup for various algorithms.

\end{abstract}


\section{Introduction}
Reinforcement Learning (RL) has shown great success in a wide range of applications including board games \cite{alphago}, strategy games \cite{alphastarblog}, energy systems \cite{chi_buildsys19}, robotics \cite{rl_robots_nn}, recommendation systems \cite{recommendation_rl}, hyperparameter selection \cite{effective_online_hyperparameter} etc. Typically, RL algorithms train by iteratively collecting the data by interacting with a simulator of the environment, and learning a model using the collected data. However, it takes a considerable amount of time to train a reinforcement learning agent to converge. This is because: 1) the speed of data collection is limited by the complexity of the environment simulator which needs to accurately represent the real world physical system; 2) the large state space needed to represent a typical real-world physical system makes it necessary to gather a large amount of data to successfully train a RL agent.  We show the training time versus the size of the state space of three popular environments used in RL training in Figure~\ref{fig:time_vs_state_space}. On Mujoco~\cite{mujoco}, which is a physics engine to simulate robotics, biomechanics, etc., it takes around 3 hours to train an agent using Pytorch \cite{pytorch} on a 4-core machine with a GTX 1060 GPU. On Atari~\cite{openai_gym}, which is a game simulator, it takes around 12 hours to train on the same machine. The state-of-the-art RL algorithm for playing Go --- AlphaGo Zero~\cite{alphago_zero} was trained on 4 TPUs \cite{tpu} for 21 days. Thus, developing faster reinforcement learning algorithms is an important research direction.



Prior work tackles this problem by deploying parallel actors that can collect data simultaneously \cite{gorila, apex, a3c, impala}. \cite{gorila} introduces a parallel framework for Deep Q Network (DQN) \cite{dqn}. 
It accelerates the training by using independent actors collecting data asynchronously. The data is stored in a shared replay buffer. 
Meanwhile, parallel learners sample data uniformly from the replay buffer and compute the gradients. 
The gradients are sent to the central parameter server \cite{parameter_server} for neural network weights update.
\cite{apex} improves the performance of \cite{gorila} by using Prioritized Replay Buffer so that important data is sampled with higher weights to accelerate the training.

In these works, Replay Buffer management becomes a limiting factor in achieving high scalability when increasing parallelism. Improving the performance of parallel Replay Buffer management via techniques such as careful data structure design or low overhead thread-level synchronization has not received much attention. In this work, we optimize the implementation of Replay Buffer management and propose a framework for generating scalable reinforcement learning implementations. The generated RL implementations are composed of parallel actors and learners executing on computing platforms such as CPUs, GPUs, or FPGAs with Replay Buffer management executing on a CPU platform. We illustrate our framework by generating RL algorithms targeting a multi-core platform. Our key contributions are summarized as follows:
\begin{itemize}
    \item We propose a new data structure for the Prioritized Replay Buffer based on $K$-ary sum tree that supports asynchronous parallel insertions, sampling and priority update.
    \item We propose a novel data layout to store the nodes of the sum tree to minimize the number of cache misses.
    \item We propose \textit{lazy writing} mechanism to minimize the thread-level synchronization overhead of various operations of the Replay Buffer.
    \item Given a hardware configuration, our framework automatically decides the number of actors and learners such that the desired ratio between the throughput of the data collection and the throughput of the learning is achieved.
    \item Our framework supports a wide range of reinforcement learning algorithms including DQN \cite{dqn}, DDQN \cite{double_q_learning}, DDPG \cite{ddpg}, TD3 \cite{td3}, SAC \cite{sac}, etc.
    \item We demonstrate the effectiveness of our framework in accelerating RL algorithms by performing experiments on CPU + GPU platforms using OpenAI \cite{openai_gym} benchmarks. Our results show that the performance of our $K$-ary sum tree based Prioritized Replay Buffer improves the baseline implementations by around 4x$\sim$100x. Our proposed synchronization optimizations improve the performance by around 2x$\sim$4.4x compared with using a global lock. By plugging our Replay Buffer implementation into existing open source reinforcement learning frameworks, we achieve 1.19x$\sim$1.75x speedup for various algorithms.
\end{itemize}



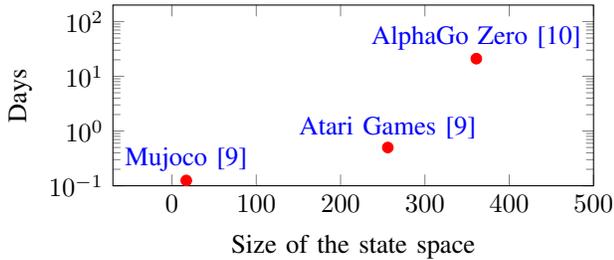
\begin{figure}
    \centering
    \begin{tikzpicture}
        \begin{axis}[
            ymode=log,
            xmin = -70, xmax = 500,
            ymin = 0.1, ymax = 200,
            xlabel={Size of the state space},
            ylabel={Days},
            width = 0.9\linewidth,
            height = 0.45\linewidth,
            nodes near coords=\pgfplotspointmeta,
            point meta=explicit symbolic
        ]
        \addplot[scatter,only marks,mark options={scale=1,fill=red,draw=red},color=blue] table [meta index=2] {
        4 0.01 CartPole
        17 0.125 Mujoco\cite{openai_gym}
        256 0.5 Atari\ Games\cite{openai_gym}
        361 21 AlphaGo\ Zero\cite{alphago_zero}
        };  
        \end{axis}
    \end{tikzpicture}
    \caption{Training time of various environments versus the size of the state space}
    \label{fig:time_vs_state_space}
\end{figure}

\section{Background}
\subsection{Markov Decision Process}\label{sec:mdp}
Reinforcement learning algorithms aim to solve Markov Decision Process (MDP) with unknown dynamics. A Markov Decision Process \cite{rl_intro} is also referred as \textit{world} or \textit{environment} in this context. An environment has five key components as follows:
\begin{itemize}
    \item State space $\mathcal{S}$: the set of all possible states in an environment. For example, in the Go game, the state space is all the possible positions of the stones.
    \item Action space $\mathcal{A}$: the set of all possible actions. For example, in the Go game, the action space is all the possible moves in the current state.
    \item System dynamics $\mathcal{P}$: the function that computes the next state given the current state and the action. 
    \item Reward function $\mathcal{R}$: the intermediate reward received by the agent when transiting from the current state to the next state.
    \item Initial state distribution $\mu$: the distribution of states where the agents will be initially at.
\end{itemize}
We define an episode as one trajectory of the agent acting from the initial state to the terminal state.
The policy $\pi(a|s)$ is defined as a stationary function that maps from the state space to the action space. The objective of reinforcement learning is to learn the policy such that the expected long-term accumulated rewards in an episode is maximized. 
\\\textbf{High level abstractions and APIs}
Reinforcement learning improves the performance of the agent by learning from the data collected from interacting with the environment. To facilitate the understanding from a system level, we introduce the high level APIs inspired from the ones used in OpenAI gym \cite{openai_gym} in Python programming language \cite{python}:
\begin{itemize}
    \item \mintinline{python}{def reset() -> S}: return a state by sampling from the initial state distribution $\mu$.
    \item \mintinline{python}{def step(a: A) -> (S, float, bool)}: return a tuple of state, reward (float type) and done signal (bool type) by taking action \textit{a}. The done signal indicates whether the current episode is finished. If the current episode is finished, call the reset to restart the episode. The environment class maintains its own internal state.
    \item \mintinline{python}{def act(s: S) -> A}: the acting function of the agent that takes the current state and outputs the action.
    \item \mintinline{python}{def learn(data: Data)}: the learning function of the agent that takes the data and updates its internal weights to improve the performance. The standard Data type contains a tuple consisting of a transition (state ($s$), action ($a$), next\_state ($s'$), reward ($r$)).
\end{itemize}

\subsection{Reinforcement Learning}
In reinforcement learning, a Replay Buffer is employed \cite{prioritized_experience_replay} to store all the data collected from the start of the training. The agent is updated using data sampled from the Replay Buffer. We show a generic paradigm of reinforcement learning algorithms in Figure~\ref{fig:generic_off_policy_rl}. 
Typical reinforcement learning algorithms include DQN \cite{dqn}, DDQN \cite{double_q_learning}, DDPG \cite{ddpg}, TD3 \cite{td3}, SAC \cite{sac}, etc. 
These algorithms only differ in how the learning is performed while the training loop is the same.

\subsection{Prioritized Replay Buffer}
To illustrate the motivation of using a Prioritized Replay Buffer \cite{prioritized_experience_replay}, we start by examining how the learning is performed in Deep Q Network (DQN) \cite{dqn}. DQN trains a Q network parameterized by $\psi$ by minimizing the following objective:
\begin{align}
    \min_{\psi}\frac{1}{N}\sum_{i=1}^{N} (Q_{\psi}(s_i,a_i) - (r_i + \gamma \max_{a_i'}Q_{\psi}(s_i',a_i')))^2
\end{align}
where $Q_{\psi}(s,a) - (r + \gamma \max_{a'}Q_{\psi}(s',a'))$ is the temporal difference (TD) error. Uniform sampling from the replay buffer to update the TD error is less effective because the sampled data may already have low TD error. Prioritized Replay Buffer \cite{prioritized_experience_replay} is proposed to mitigate this problem by assigning a priority to each data item using the absolute value of the TD error:
\begin{align}
    P(i) = |Q_{\psi}(s_i,a_i) - (r_i + \gamma \max_{a_i'}Q_{\psi}(s_i',a_i'))|
\end{align}
where $P(i)$ denotes the priority of data $i$.
Then, the data is sampled according to the probability proportional to the priority. To fix the bias introduced by the prioritized sampling, importance weights are computed as $w(i) = (\frac{1}{N}\cdot\frac{\sum_{i}P(i)}{P(i)})^{\beta}$, where $w(i)$ denotes the importance weights for data $i$ and $\beta$ is a hyper-parameter. The learning step of DQN using a Prioritized Replay Buffer is:
\begin{align}
    \min_{\psi} \frac{1}{N}\sum_{i=1}^{N} w(i)\cdot(Q_{\psi}(s_i,a_i) - (r_i + \gamma \max_{a_i'}Q_{\psi}(s_i',a_i')))^2
\end{align}
After each update, the new priority is stored in the Replay Buffer. A complete training process is shown in Algorithm~\ref{alg:off_policy_rl}. Other algorithms follow the same structure and only differ slightly in the technique used to update the Q function.

\begin{figure}
    \centering
    \includegraphics[width=\linewidth]{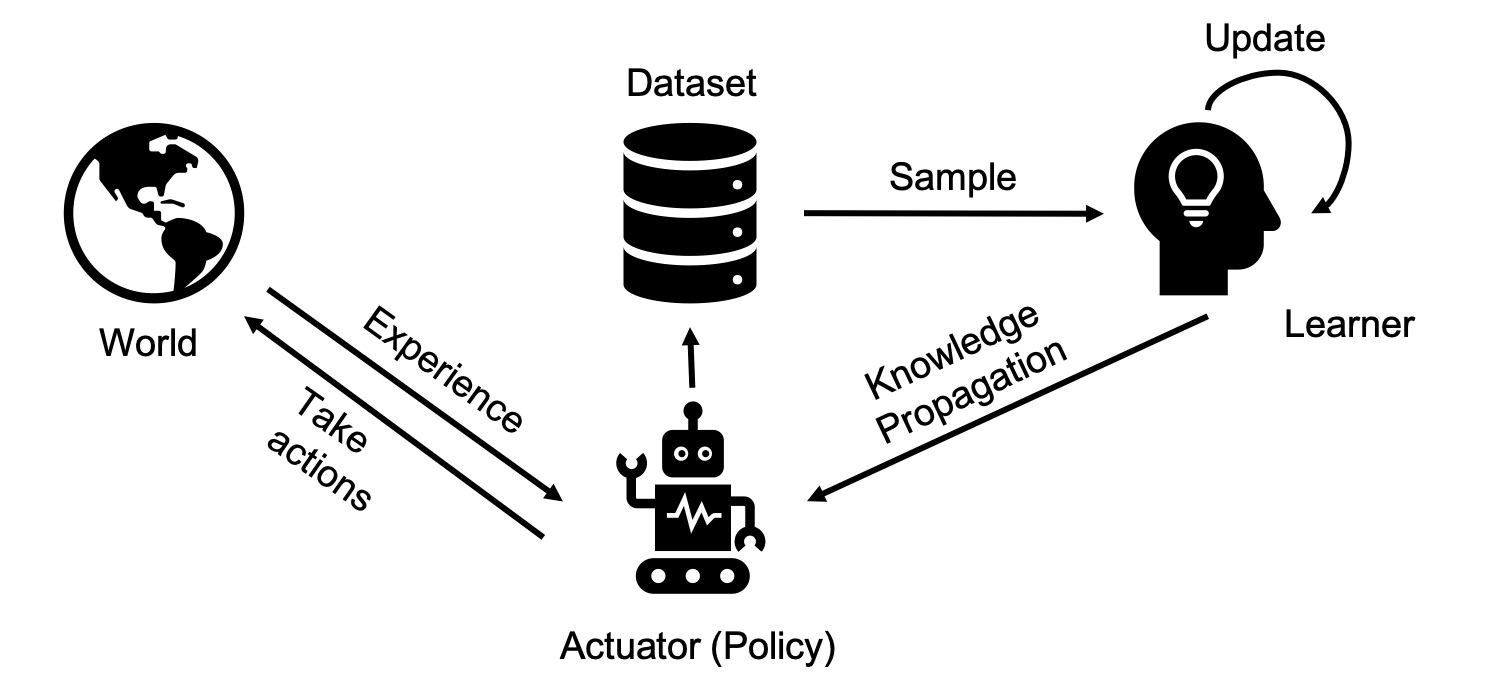}
    \caption{A generic paradigm of reinforcement learning algorithms}
    \label{fig:generic_off_policy_rl}
\end{figure}

\begin{algorithm}[!t]
    \caption{Generic Reinforcement Learning}
    \label{alg:off_policy_rl}
    \begin{algorithmic}[1]
        \State \textbf{Input:} Environment \textit{env}, Agent $\pi_{\theta}$, Replay buffer $\mathcal{B}$.
        \State \textbf{Output:} Trained agent $\pi_{\theta}$.
        \For{$i$ = 1; $i$ $\leq$ iterations; $i++$}
            \If{done}
                \State obs = \textit{env}.reset(); \Comment{Episode terminates}
            \EndIf
            \State action = agent.act(obs); \Comment{Agent select action}
            \State next\_obs, reward, done = \textit{env}.step(action); \Comment{Actuator}
            \State $\mathcal{B}$.insert(obs, action, next\_obs, done);
            \State obs = next\_obs;
            \If{$i$ \% update\_interval == 0}
                \State index, data = $\mathcal{B}$.sample(batch\_size);
                \State priority = $\mathcal{B}$.get\_priority(index);
                \For{$i$ in index}
                    \State $w(i) = (\frac{1}{N}\cdot\frac{\sum_{i}P(i)}{P(i)})^{\beta}$; \Comment{importance weights}
                \EndFor
                \State new\_priority = agent.learn(data, is);
                \State $\mathcal{B}$.update\_priority(index, new\_priority);
            \EndIf
        \EndFor
    \end{algorithmic}
\end{algorithm}

\section{Related Work}
\subsection{Parallel and Distributed Reinforcement Learning}
Existing works that aim to improve the execution time of Reinforcement Learning (RL) algorithms focus on increasing the parallelism by increasing the number of actors and learners. GORILA \cite{gorila} proposes the first parallel architecture of DQN \cite{dqn} to play Atari games \cite{openai_gym}. 
They employ independent actors and learners in parallel with a global parameter server. 
Our method follows the general architecture of GORILA \cite{gorila} at a high level and proposes detailed data structures and thread-level synchronization mechanism to maximize the scalable performance. A3C \cite{a3c} uses asynchronous actors to collect the data and update the agent using actor critic algorithms without using a Replay Buffer. Due to synchronization overhead, A3C doesn't scale very well. IMPALA \cite{impala} relaxes the synchronization overhead of A3C by using importance sampling. RLlib \cite{ray_rllib} proposes abstractions for distributed reinforcement learning for software developers built on top of the Ray library \cite{ray_rllib} written in Python \cite{python}. PAAC \cite{paac} proposes parallel advantaged actor critic. They synchronize the actors after every environmental step. This significantly slows down the entire system. In contrast, our actors act independently in parallel.
\cite{map_reduce_parallel_rl} proposes parallel reinforcement learning using popular MapReduce \cite{map_reduce} framework with linear function approximation.
\cite{parallel_rl} proposes to use parallel actors to learn in tabular MDP while our method can tackle general continuous space MDP with neural network policies. 

A key bottleneck in these works is the management of Replay Buffer. Thread-level synchronization overheads and irregular memory accesses while accessing the Replay Buffer lead to poor scalability when parallelism is increased by adding more hardware resources. Ape-X \cite{apex} proposes distributed Prioritized Replay Buffer with parallel actors and a single learner to accelerate reinforcement learning algorithms on large scale clusters. However, to the best of our knowledge, our approach is the first to explicitly focus on improving the efficiency of Replay Buffer management on multi-core platforms by developing a novel data structure and low overhead thread-level synchronization mechanisms to enable high throughput parallel Replay Buffer management.

In addition to these works, specialized hardware designs to accelerate reinforcement learning have also emerged recently. \cite{trpo_fpga} proposes customized Pearlmutter Propagation on FPGAs to accelerate conjugate gradient method used in TRPO \cite{trpo}. \cite{ppo_fpga} proposes a systolic-array based architecture on FPGAs to accelerate PPO \cite{ppo}. However, these works do not require the use of Replay Buffer.

\subsection{Parallel Stochastic Gradient Descent}
We also review techniques for performing parallel stochastic gradient descent as it is used in our learner implementation. \cite{parameter_server} proposes parameter server to facilitate parallel stochastic gradient descent. Each worker samples a batch of data, computes the gradients and send them to the central parameter server. The parameter server aggregates the gradients and performs the update. The workers then pull the updated weights from the parameter server. \cite{AsyncPSGD} proposes asynchronous stochastic gradient descent to reduce the negative impact of asynchrony with general convergence time bounds. For simplicity, we adopt the parameter server \cite{parameter_server} framework and leave advanced asynchronous methods for future work.

\begin{algorithm}[!t]
    \caption{Key operations of the N-ary sum tree.}
    \label{alg:n_nary_sum_tree_func}
    \begin{algorithmic}[1]
      \Function{updateValue}{idx, value}
          \State node\_idx = \Call{convertToNodeIdx}{idx};
          \State $\Delta$ = value - \Call{getValue}{node\_idx};
          \While{!\Call{isRoot}{node\_idx}}
              \State new\_value = \Call{getValue}{node\_idx} + $\Delta$;
              \State \Call{SetValue}{node\_idx, new\_value};
              \State node\_idx = \Call{getParent}{node\_idx};
          \EndWhile
      \EndFunction
      \\
      \Function{getPrefixSumIdx}{prefixSum}
          \State node\_idx = \Call{getRoot}{\null};
          \While{!isLeaf(node\_idx)}
              \State node\_idx = \Call{getLeftChild}{node\_idx};
              \State partialSum = 0;
              \For{$i$ = 0; $i$ $<$ fan\_out; $i++$}
                  \State sum = partialSum + \Call{getValue}{node\_idx};
                  \If{sum $\geq$ prefixSum}
                      \State break;
                  \EndIf
                  \State partialSum = sum;
                  \State node\_idx = \Call{getNextSibling}{node\_idx};
              \EndFor
              \State prefixSum = prefixSum - partialSum;
          \EndWhile
          \State idx = \Call{convertToDataIdx}{node\_idx};
          \State \Return idx;
      \EndFunction
    \end{algorithmic}
\end{algorithm}

\section{Parallel Prioritized Replay Buffer}\label{sec:parallel_per}
In this section, we discuss in detail the design of our Prioritized Replay Buffer that supports parallel actors and learners. We start by introducing the key operations that need to be supported.
\subsection{Operations}\label{sec:parallel_per_ops}
\subsubsection{Insertion}
Given a new data item $x$, find the next available index $i$ and insert $x$ at $i$. If the Replay Buffer is full, find the index using the eviction policy. Set the priority at index $i$ to $P(i)=P_{\max}$, where $P(i)$ is the priority at index $i$ and $P_{\max}$ is the maximum priority in the Replay Buffer. The most common eviction policy used in existing implementations is First-in-first-out (FIFO).
\subsubsection{Sampling}
Sample a data item $x_i$ according to the probability distribution $Pr(i)=\nicefrac{P(i)}{\sum_{i}P(i)}, i=1,2,\ldots, N$, where $N$ is the size of the Replay Buffer. To do so, we first sample $x$ from uniform distribution $U(0,1)$. Then, we compute the cumulative density function (cdf) as $cdf(i)=\sum_{j=1}^{i}Pr(j), i=1,2,\ldots, N$. Finally, the sampled index $i=cdf^{-1}(x)$. Mathematically, this is equivalent to finding the minimum index $i$, such that the \textbf{prefix sum} of the probability from 0 to $i$ is greater than or equal to $x$:
\begin{align}
    \min_{i}\sum_{j=1}^{i}Pr(i)\geq x \Rightarrow \min_{i}\sum_{j=1}^{i}P(i)\geq x\cdot \sum_{j=1}^{N}P(j)
    \label{eq:prefix_sum}
\end{align}
\subsubsection{Priority retrieval}
Return the priority at index $i$.
\subsubsection{Priority update}
Update the priority at index $i$.

\subsection{Frequency of the Operations and Runtime Requirements} 
As shown in Algorithm~\ref{alg:off_policy_rl}, the insertion is executed once per iteration. The sampling, priority retrieval and priority update are executed once every update\_interval. Directly storing the priority in an array incurs a runtime complexity of $\Theta(N)$ for sampling and $\Theta(1)$ for priority retrieval and priority update. Directly storing the prefix sum in an array incurs a runtime of $\Theta(\log N)$ in sampling, $\Theta(1)$ in priority retrieval and $\Theta(N)$ in priority update. Based on the frequency of the operations, both these implementations incur a overall runtime complexity of $\Theta(N)$. In this paper, we proposed to use $K$-ary sum tree to implement the Prioritized Replay Buffer to achieve $\Theta(\log N)$ runtime complexity for both sampling and priority update and thus for the entire implementation.

\subsection{$K$-ary Sum Tree}\label{sec:sum_tree}
We show an example of a $K$-ary sum tree for $K=4$ in Figure~\ref{fig:sum_tree}. Each node has $K$ child nodes. The value stored in the parent node is the sum of all the values stored in the child nodes. The leaf nodes hold the actual priorities. 

\begin{figure}
    \centering
    \includegraphics[width=\linewidth]{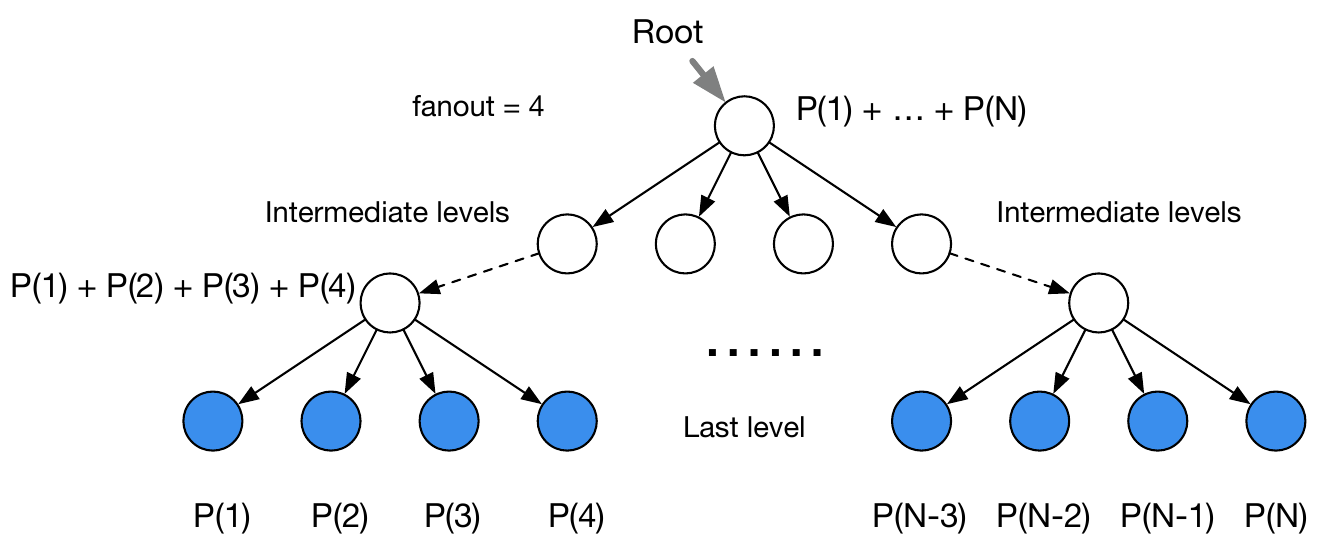}
    \caption{The overall structure of a 4-ary sum tree.}
    \label{fig:sum_tree}
\end{figure}

\begin{figure}
    \centering
    \includegraphics[width=\linewidth]{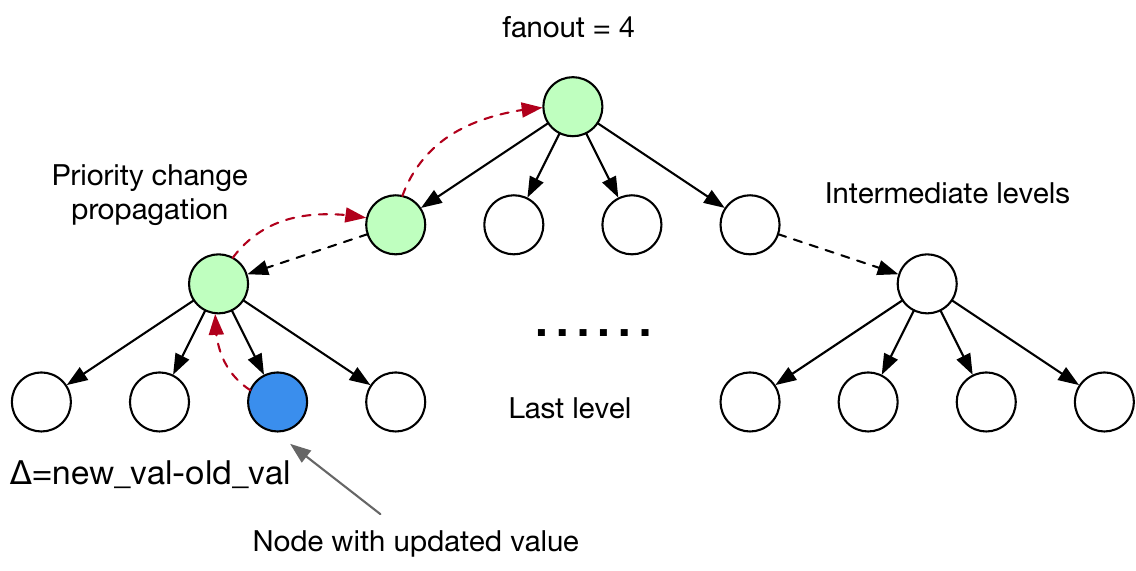}
    \caption{Illustration of the process when updating the value in the $K$-ary sum tree with fanout=4 as shown in Algorithm~\ref{alg:n_nary_sum_tree_func}. The blue node denotes the leaf node that holds the priority. The green nodes denote the intermediate sums that are updated by propagating the change of the priority from the leaf to the root. The red dotted arrow shows the direction of the value propagation.}
    \label{fig:update_value}
\end{figure}

\begin{figure*}
    \centering
    \includegraphics[width=\linewidth]{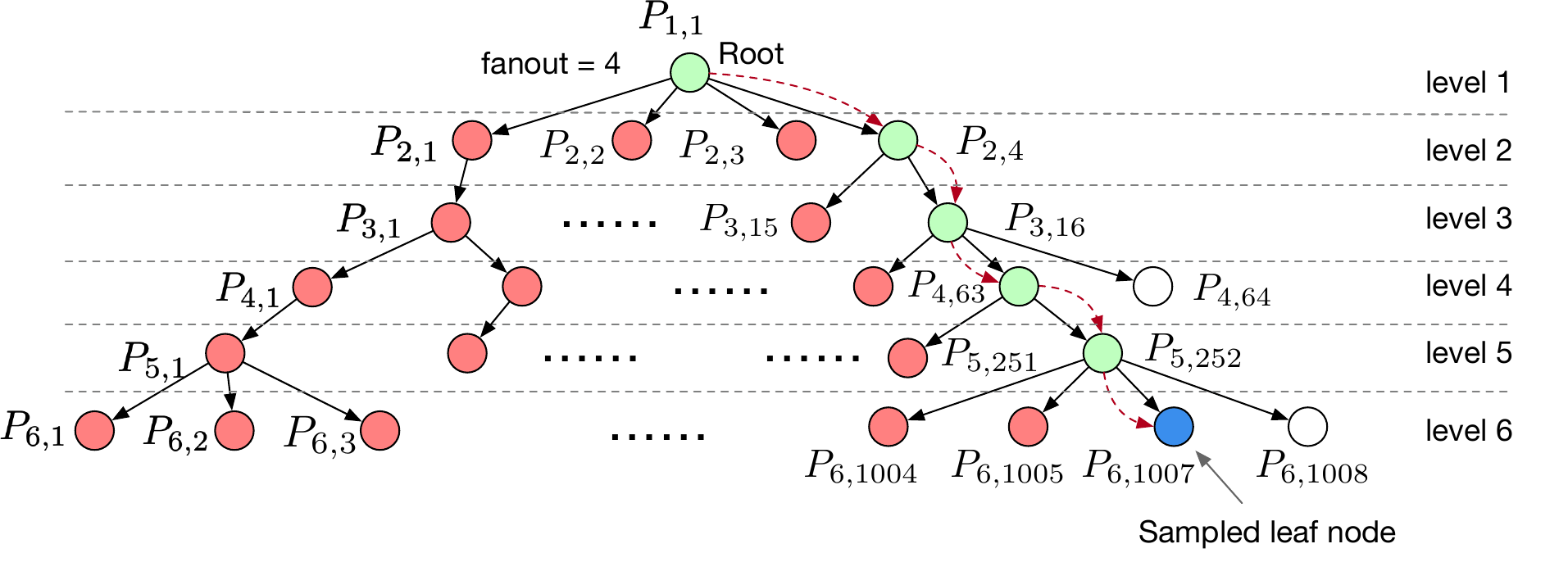}
    \caption{Illustration of the process when sampling index according to the priority in the $K$-ary sum tree with fanout=4 as shown in Algorithm~\ref{alg:n_nary_sum_tree_func}. Starting from the root node, the green nodes denote the cutoff node during traversal and the blue node denotes the leaf node sampled. The red dotted arrow shows the direction of the tree traversal.}
    \label{fig:get_prefix_sum}
\end{figure*}

\subsubsection{Priority retrieval}
In order to obtain the priority for the index $i$, we create an array of pointers, each pointing to its corresponding leaf node that holds the priority value. Thus, priority retrieval using $K$-ary sum tree requires $\Theta(1)$ time.

\subsubsection{Priority update}
To update the priority of index $i$, we first obtain the leaf node holding the priority. We compute the change of the priority by subtracting the old value from the new value. Then, we propagate the change of the priority from the leaf node to the root node by traversing along the parent nodes. We show a detailed function in Algorithm~\ref{alg:n_nary_sum_tree_func} and an example in Figure~\ref{fig:update_value}. It is easy to see that this operation requires $\Theta(\log_K N)$ time.

\subsubsection{Prefix sum index computation}
Given a randomly sampled number $x\sim U(0,1)$, the objective is to compute index $i=\min_{i}\sum_{j=1}^{i}P(i)\geq x\cdot \sum_{j=1}^{N}P(j)$ as discussed in Section~\ref{sec:parallel_per_ops}. The sum of all the priorities in the Replay Buffer $\sum_{j=1}^{N}P(j)$ can be computed in $\Theta(1)$ by simply retrieving the value stored in the root node. To design an algorithm that obtains the target index, we start by proving Lemma~\ref{lemma:index_existence} and theorem~\ref{the:sum_tree_prefix}:

\begin{lemma}
    \label{lemma:index_existence}
    Let the value of the $i$-$th$ node at level $m$ be $P_{i,m}$. Assume the height of the tree is $H$. Then, at level $1\leq m\leq H$, there exists index $j$, $1 \leq j \leq K^{m-1}$, such that $\sum_{i=1}^{j}P_{m, i}\geq x\cdot \sum_{i=1}^{N}P(i)$, for any $x\in (0, 1)$.
\end{lemma}
\begin{proof}
    According to the definition, the leaf node holds the priority value. Thus, $P_{i, H}=P(i), \forall i=1,2,\ldots, K^{H-1}$. Since $x\in (0, 1)$, we obtain $x\cdot \sum_{i=1}^{N}P(i)\leq \sum_{i=1}^{N}P(i) \leq \sum_{i=1}^{K^{H-1}}P_{i,H}$. Because the priority values are non-negative, there must exist index $j$, $1 \leq j \leq K^{H-1}$, such that $\sum_{i=1}^{j}P_{H, i}\geq x\cdot \sum_{i=1}^{N}P(i)$. According to the property of the sum tree, the value of the parent is the sum of all its children. Thus, $\sum_{i=1}^{K^{H-1}}P_{H, i}=\sum_{i=1}^{K^{m-1}}P_{m, i}, \forall m=1, 2, \cdots, H-1$. Therefore, the same argument holds for each level. This concludes the proof for Lemma~\ref{lemma:index_existence}.
\end{proof}
\begin{theorem}
    \label{the:sum_tree_prefix}
    Let $j_m=\min_{j'} \sum_{i=1}^{j'}P_{m, i}\geq x\cdot \sum_{i=1}^{N}P(i)$. Then, $j_m$ is the parent node of $j_{m+1}$, $\forall m=1,2,\cdots H-1$.
\end{theorem}
\begin{proof}
    The child nodes of index $j$ at level $m$ are $K\cdot(j-1)+1,\cdots, K\cdot j$ at level $m+1$. According to the definition of the sum tree and the property of $j_m$, we obtain $\sum_{i=1}^{j_m-1}P_{m,i}=\sum_{i=1}^{K\cdot(j_m-1)}P_{m+1, i}< x\cdot \sum_{i=1}^{N}P(i)$. Thus, the index of the cutoff node at level $m+1$ must be $j_{m+1}\geq K\cdot(j_m-1)+1$. Noticing that $\sum_{i=1}^{j_m}P_{m,i}=\sum_{i=1}^{K\cdot(j_m)}P_{m+1, i}\geq x\cdot \sum_{i=1}^{N}P(i)$. Thus, the index of the cutoff node at level $m+1$ satisfies $j_{m+1}\leq K\cdot j_m$. Combining $K\cdot(j_m-1)\leq j_{m+1}\leq K\cdot j_m$, we obtain $j_m$ is the parent node of $j_{m+1}$.
\end{proof}
We refer such node $j_m$ as the \textit{cutoff node} at level $m$. The goal of sampling is to find the index of the cutoff node at the last level of the tree. 
According to Theorem~\ref{the:sum_tree_prefix}, the cutoff node at level $m$ is the parent of the cutoff node at level $m+1$. 
Therefore, we can start from the root node and perform the search only using the child nodes. To obtain which child node is the cutoff node, we maintain a cumulative sum of all the nodes left to the cutoff at each level. Please refer to Algorithm~\ref{alg:n_nary_sum_tree_func} for details.
We also illustrate an example of the process in Figure~\ref{fig:get_prefix_sum}, where $K=4$ and $H=6$.

\subsubsection{Data layout}
Maintaining the explicit tree data structure using pointers significantly degrades the cache performance of modern CPUs. In this work, we implement the tree data structure implicitly using an array as shown in Figure~\ref{fig:data_layout}. The sampling process requires traversing all the nodes under the same parent. To maximize the cache performance, it is desired that each group of child nodes under the same parent is cache aligned. Assume that one cacheline can store $C$ nodes, then we choose $K$, such that $K\%C=0$. We pad the root node with $K-1$ so that it is also cache aligned.

\begin{figure}
    \centering
    \includegraphics[width=\linewidth]{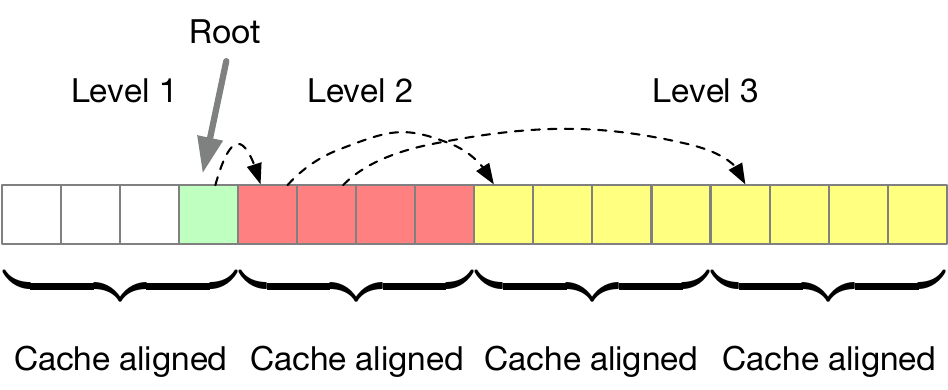}
    \caption{Data layout of the proposed sum tree with $K=4$. The black arrow shows the first child of the parent node. The color indicates the level of the node in the tree.}
    \label{fig:data_layout}
\end{figure}

\subsubsection{Theoretical performance analysis}
\paragraph{Space complexity}
The space complexity is proportional to the number of nodes in the tree. Assume the size of the Replay Buffer is $N$, which is equal to the number of nodes in the last level of the tree. Thus, the total number of nodes in the tree is: $\Theta(\frac{K^H-1}{K-1})=\Theta(\frac{K^{H-1}\cdot K-1}{K-1})=\Theta(\frac{N\cdot K-1}{K-1})=\Theta(N+\frac{N-1}{K-1})$. Clearly, as $K$ increases, the space complexity reduces due to the decrease of the number of intermediate nodes.

\paragraph{Runtime complexity}
It is clear that the priority retrieval runs in $\Theta(1)$ and priority update runs in $\Theta(\log_K N)$. For prefix sum index computation, the loops runs $H=\ceil{\log_K N} + 1$ times. The memory access inside loop has $\nicefrac{K}{C}$ cache misses and $K\cdot(1-\nicefrac{1}{C})$ cache hit, where $C$ is the number of nodes in one cacheline. Thus, the time complexity of prefix sum index computation is $\Theta((\log_K N+1)(T_{miss}\cdot\frac{K}{C}+T_{hit}\cdot K\cdot(1-\nicefrac{1}{C})))$, where $T_{miss}$ is the execution time of one cache miss and $T_{hit}$ is the execution time of one cache hit. Note that this function has a local minimum in terms of $K$. In practice, we profile the performance of various $K$ values based on the size of the cacheline and choose the one that yields the best performance.

\subsection{Thread-safe Prioritized Replay Buffer}\label{sec:thread_safe_prb}
In order to support parallel actors and learners, it is crucial to design thread-safe prioritized Replay Buffer. We summarize the resource utilization of various operations in Table~\ref{table:resource_util}. We design the thread-safe prioritized replay buffer using locking mechanism such that the duration of holding a lock is minimized. 

\subsubsection{Synchronization of the sum tree}
We use two locks to synchronize the sum tree: one to synchronize the read/write of the last level of the tree and the other to synchronize the read/write of all the levels. A detailed procedure of priority update and priority retrieval is shown in Algorithm~\ref{alg:sync_replay_buffer}. Using this technique, reading of the priority value and updating of the intermediate levels of the sum tree can be executed in parallel. Note that it will cause inconsistencies if we acquire the global\_tree\_lock after releasing the last\_level\_lock lock when two priority update queries arrive at the same time.

\subsubsection{Synchronization of insertion and sampling}
During insertion, the Replay Buffer finds an available index. Then, it writes the data to the storage and updates the priority to the maximum priority in the Replay Buffer. Compared with index searching and priority update, data writing takes more time due to explicit copy of the memory data. Thus, it is important not to hold the lock while performing the data writing. To do so, we propose \textit{lazy writing}: 1) We set the priority to zero atomically; ii) we perform data writing; iii) we reset the priority to the maximum priority in the Replay Buffer atomically. Since the priority is zero during data writing, it will never be sampled. This makes sampling only needs to synchronize prefix sum index computation. A detailed procedure is shown in Algorithm~\ref{alg:sync_replay_buffer}.

\begin{algorithm}[!t]
    \caption{Synchronization of the Prioritized Replay Buffer}
    \label{alg:sync_replay_buffer}
    \begin{algorithmic}[1]
        \Function{PriorityUpdate}{idx, new\_priority}
            \State \Call{Acquire}{global\_tree\_lock};
            \State \Call{Acquire}{last\_level\_lock};
            \State \Call{UpdateLastLevel}{\null};
            \State \Call{Release}{last\_level\_lock};
            \State \Call{UpdateIntermediateLevel}{\null};
            \State \Call{Release}{global\_tree\_lock};
        \EndFunction
        \\
        \Function{PriorityRetrieval}{idx}
            \State \Call{Acquire}{last\_level\_lock};
            \State priority = \Call{getPriority}{idx};
            \State \Call{Release}{last\_level\_lock};
            \State \Return priority;
        \EndFunction
        \\
        \Function{Insert}{idx, data}
            \State \Call{UpdatePriority}{idx, 0};
            \State \Call{WriteToStorage}{idx, data};
            \State \Call{UpdatePriority}{idx, max\_priority};
        \EndFunction
        \\
        \Function{Sample}{\null}
            \State \Call{Acquire}{global\_tree\_lock};
            \State idx, priority = \Call{getPrefixSumIndex}{\null};
            \State \Call{Release}{global\_tree\_lock};
            \State \Return idx, priority;
        \EndFunction
    \end{algorithmic}
\end{algorithm}

\begin{table}[!t]
    \centering
    \caption{Resource utilization of various operations}
    \vspace{-1em}
    \begin{tabular}{cc}
        \toprule
        Operations & Resource Utilization  \\
        \midrule
        Insertion & modify the entire tree, modify the storage \\
        Sampling & access the entire tree, access the storage \\
        Priority retrieval & access the last level of the tree \\
        Priority update & modify the entire tree \\
        \bottomrule
    \end{tabular}
    \label{table:resource_util}
\end{table}

\subsubsection{Write after read vs. read after write}
The parallelism of the priority update and the data sampling causes data dependency issues: the same data is sampling using the old priority before the new priority gets updated (write after read). Mathematically, only read after write is valid and write after read produces inconsistent results. However, it has little impact in practice as neural network training is stochastic in nature and robust to such transient inconsistencies.

\section{Overall Framework}
\begin{figure*}
    \centering
    \includegraphics[width=\linewidth]{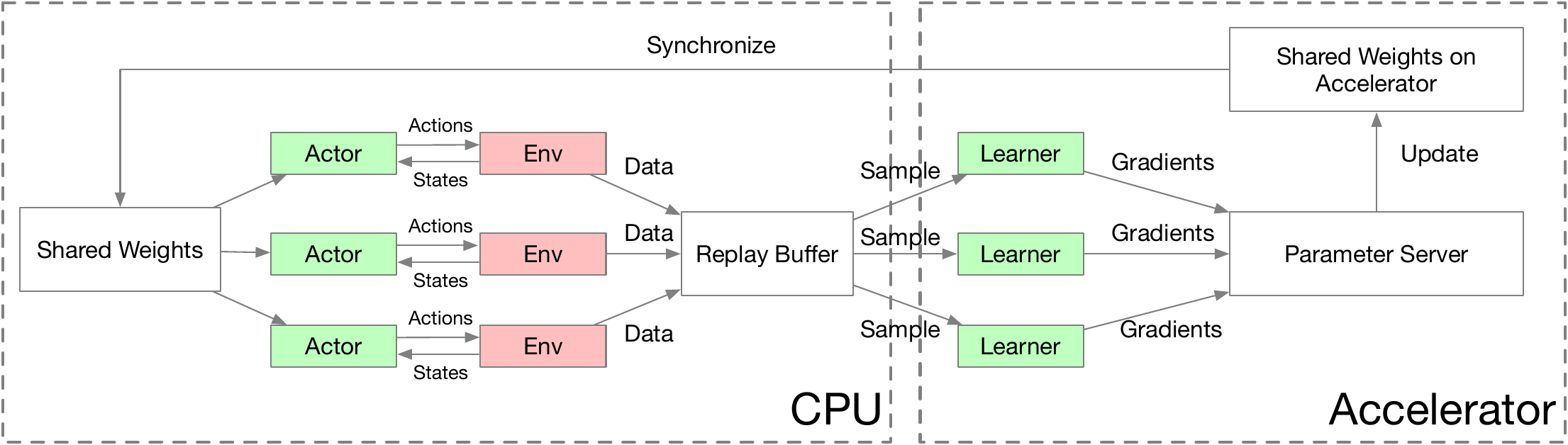}
    \caption{Overall system architecture}
    \label{fig:overall_system}
\end{figure*}

The overall system architecture is shown in Figure~\ref{fig:overall_system}. We employ parallel actors to collect data and parallel learners to compute the gradients for neural network weights update. 

\subsection{Asynchronous Actors}
Asynchronous actors collect the data simultaneously by interacting with their own instance of the environment using the shared weights. The data is then added to the Replay Buffer. It is worth noting that no synchronization is required because the inference doesn't alter the weights.

\subsection{Parallel Learners}
Deploying parallel actors increases the throughput of data collection. In order to increase the throughput of the learning, we employ parallel learners with a central parameter server \cite{parameter_server}. Each learner independently samples one batch of data from the Replay Buffer and computes the sub-gradients. The parameter server aggregates the gradients and updates the weights.

\subsection{Framework Specification}
Our framework supports a wide range of reinforcement learning algorithms including DQN \cite{dqn}, DDQN \cite{double_q_learning}, DDPG \cite{ddpg}, SAC \cite{sac}, TD3 \cite{td3} and so on. The target platform of our framework is processor + accelerator platforms, where the processor is the CPU the accelerator is either the GPU or the FPGA. The input of our framework includes:
\begin{itemize}
    \item The overall throughput of the data collection vs. the number of CPU cores.
    \item The overall throughput of the data consumption vs. the number of CPU cores.
    \item Total number of cores in the CPU.
\end{itemize}
The throughput of the data collection by a single actor is affected by i) the time of a single environment step function defined in Section~\ref{sec:mdp}; ii) The specifications of the neural networks used in the actors including the architecture (fully-connected vs. convolution networks), the size of each layer, etc. iii) the speed of the processor. The throughput of a single learner is affected by i) the reinforcement learning algorithm; ii) the optimizer iii) the speed of the accelerator.

\subsection{Design Space Exploration}\label{sec:design_space_exp}
The objective of is to choose the number of actor threads and the number of learner threads such that the ratio between the throughput of the data collection vs. data consumption is the same as the single thread implementation (update\_interval denoted in Algorithm~\ref{alg:off_policy_rl}). In order to obtain the allocation of the cores, we profile the overall throughput of the data collection vs. the number of CPU cores and denote the curve as $f_{a}(x)$, where $x$ is the number of cores. Similarly, we profile the overall throughput of the data consumption vs. the number of CPU cores and denote the curve as $f_{l}(x)$. Let the total number of CPU cores be $M$. Then, the design space exploration is the solution of equation~\ref{eq:design_space}:
\begin{align}
    \label{eq:design_space}
    f_{a}(x_a) &= \text{update\_interval}\times f_{l}(x_l) \nonumber\\
    x_a + x_l &\leq M
\end{align}
where $x_a$ and $x_l$ is the allocated number of cores for actors and learners, respectively.
If the parallel actors and/or learners are deployed on an accelerator such as GPU or FPGA, instead of CPU, profiling similar to the one described above can be used to perform the design space exploration.

\section{Experiments}
Our experiments aim to answer the following questions:
\begin{enumerate}
    \item How does our proposed Prioritized Replay Buffer compare against existing baseline approaches? (See Section~\ref{sec:comp_baselin})
    \item Does the performance of our proposed Prioritized Replay Buffer follow the theoretical analysis in Section~\ref{sec:sum_tree} in terms of the fanout size $K$? (See Section~\ref{sec:perf_prioritized})
    \item How does our proposed locking mechanisms for the prioritized replay buffer reduce the synchronization overhead compared with using a global lock? (See Section~\ref{sec:perf_prioritized})
    \item What is the performance improvement when plugging in our prioritized replay buffer implementation into existing RL frameworks? (See Section~\ref{sec:comp_existing})
\end{enumerate}

\subsection{Experimental Setup}
We conduct our experiments on 56-core Intel(R) Xeon(R) Gold 5120 CPUs with 128GB DDR4 memory and a Nvidia TITAN Xp GPU with 12GB GDDR6 memory as the accelerator. We implement the synchronization mechanism using pthreads \cite{pthread} and the training of neural networks using LibTorch \cite{pytorch}. We test our framework on reinforcement learning algorithms including DQN \cite{dqn} and DDPG \cite{ddpg}. DQN targets at discrete action space while DDPG and SAC targets at continuous action space. We use \textit{LunarLander-v2} \cite{openai_gym} environment to test the algorithms. In all our experiments, the desired ratio between the throughput of the data collection and the data learning (update\_interval) is set to 1.

\subsection{Baseline Approach}
In this work, we use RLlib \cite{ray_rllib} as our baseline. RLlib is an open source implementation of parallel and distributed framework for training reinforcement learning agents written in Python \cite{python}. For fair comparison, we use the same amount of cores when running the experiments. We also compare with the Prioritized Replay Buffer implementation in open source RL framework tianshou \cite{tianshou}.

\subsection{Comparison with Baseline Approaches}\label{sec:comp_baselin}
\begin{figure}
    \centering
    \includegraphics[width=\linewidth]{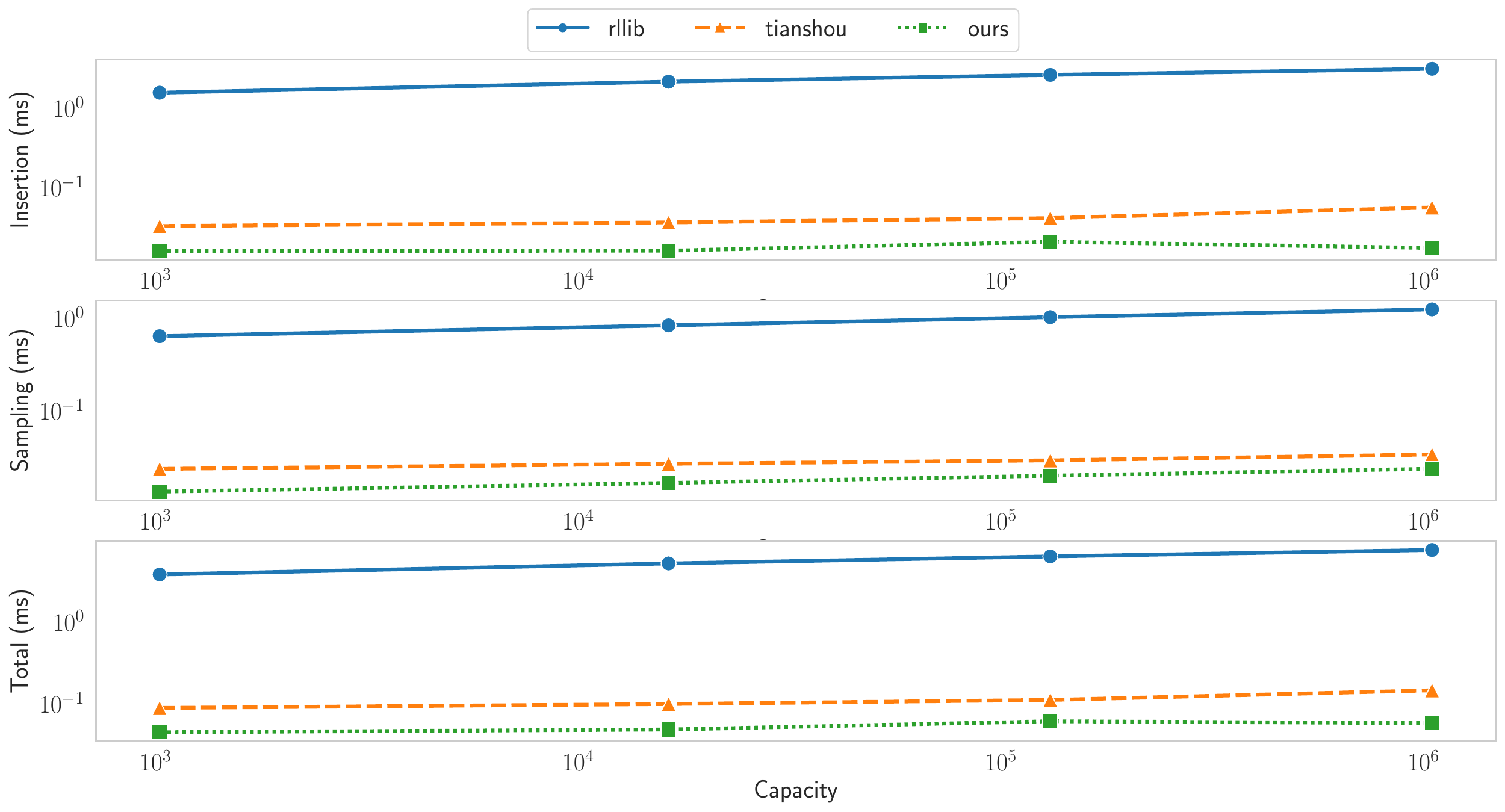}
    \caption{Comparison with the baseline approach}
    \label{fig:baseline_comp}
\end{figure}

We show the latency of insertion and sampling of the Replay Buffer with various sizes in Figure~\ref{fig:baseline_comp}. We compare our $K$-ary sum tree based implementation with RLlib \cite{ray_rllib} and tianshou \cite{tianshou}. Overall, our approach reduces the total latency by around 4x compared with tianshou \cite{tianshou} and around 100x compared with RLlib \cite{ray_rllib}. Note that the latency of the Prioritized Replay Buffer operations in  RLlib increases in linear while the latency of our implementation increases in sub-linear. This suggests our $K$-ary based Prioritized Replay Buffer has better scalability compared with \cite{ray_rllib}.

\subsection{Performance of the Prioritized Replay Buffer}\label{sec:perf_prioritized}
\begin{figure}
    \centering
    \includegraphics[width=\linewidth]{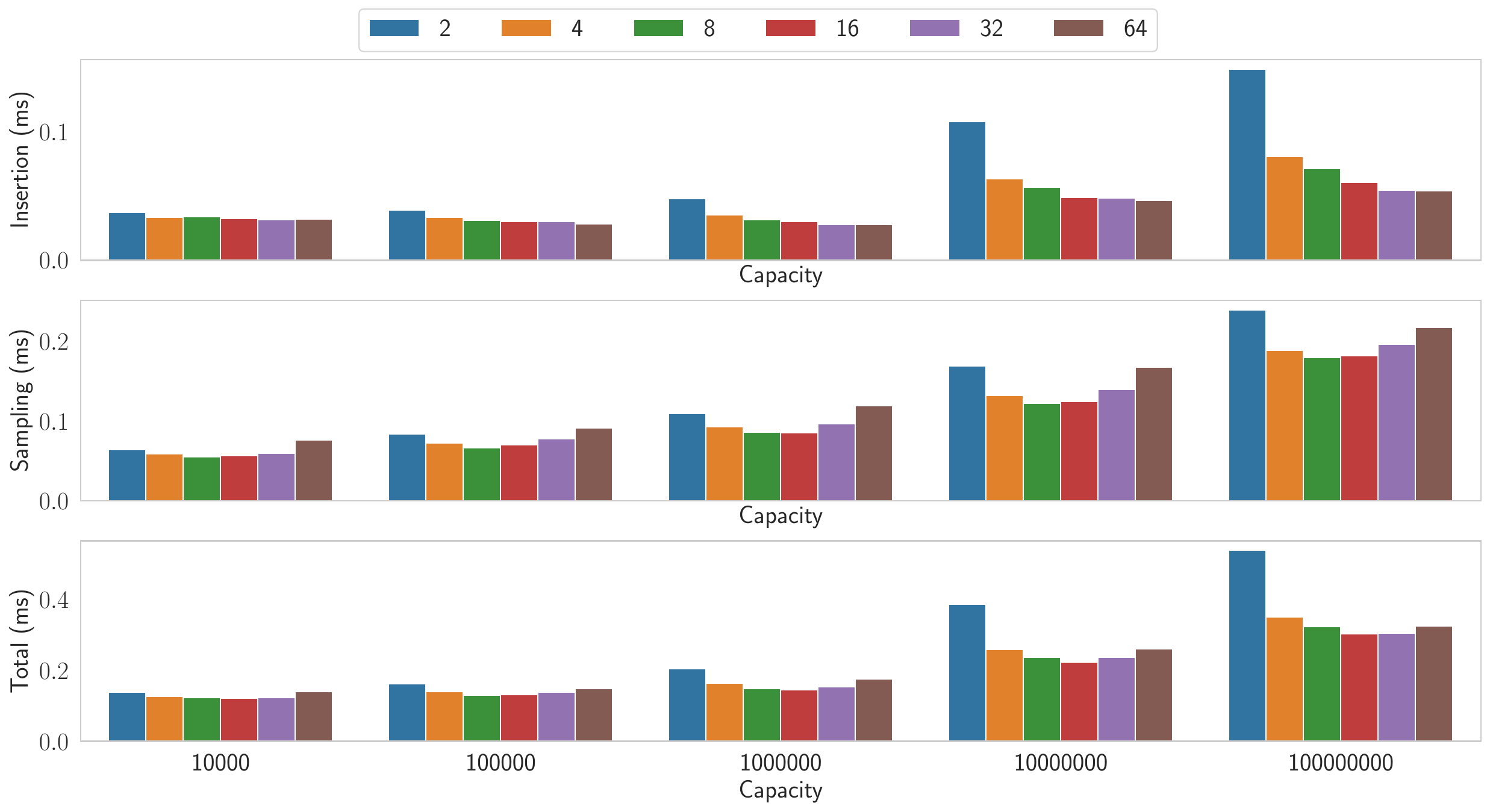}
    \caption{Latency of various Prioritized Replay Buffer operations with various fannout $K$}
    \label{fig:k_ary_sum_tree}
\end{figure}
\subsubsection{Effect of fanout $K$}
In order to answer question 2, we show the latency of insertion and sampling of various $K$ in Figure~\ref{fig:k_ary_sum_tree}. We also vary the capacity of the Replay Buffer to demonstrate the scalability. First, we observe that the latency of insertion decreases when $K$ increases. This matches our theoretical performance analysis because the latency is proportion to the height of the tree. The height of the tree decreases when $K$ increases. Second, we observe that the latency of sampling first decreases to a local minimum and then increases as $K$ increases. This also matches with our theoretical analysis because as $K$ increases, the latency increase of search over each level starts to dominate the latency decrease with fewer number of levels. In order to choose the optimal $K$, we simply perform profiling of insertion and sampling to obtain the total latency. In our experimental machine, $K=16$ yields the best result.

\subsubsection{Effect of synchronization optimization}

\begin{figure}
    \centering
    \includegraphics[width=\linewidth]{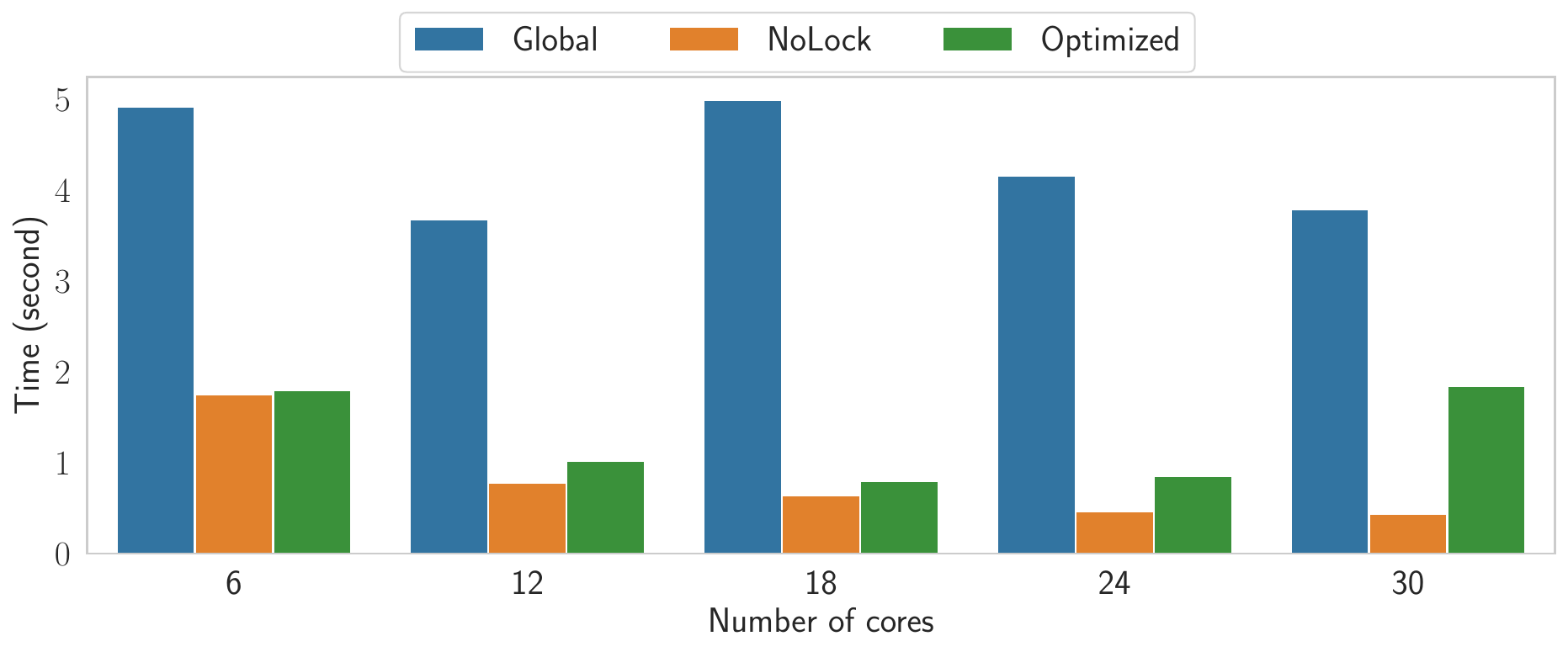}
    \caption{Execution time of the Prioritized Replay Buffer with various synchronization methods}
    \label{fig:scalability}
\end{figure}
In order to answer question 3, we show the execution time of 5000 iterations versus the number of CPU cores using a global lock, no lock and our proposed synchronization optimization in Section~\ref{sec:thread_safe_prb}. Although the results of computations without using lock are wrong, it provides an upper bound on the performance. We observe that our proposed thread-level synchronization enables 1.01x$\sim$5x increase of the execution time compared with the minimum achievable execution time; and achieves 2x$\sim$5x improvement against using a global lock. Moreover, our design scales well in the number of CPU cores.

\subsection{Performance improvement of existing frameworks using our proposed replay buffer}\label{sec:comp_existing}
In order to show the superiority of our proposed Prioritized Replay Buffer, we write a Python binding of the C++ implementation and plug it into existing open source RL framework RLlib \cite{ray_rllib}. We show the latency of each training step of two RL algorithms in Figure~\ref{fig:library_perf}. 
Overall, we achieve 1.19x$\sim$ 1.75x performance improvement using various CPU cores.
The speedup decreases as the number of CPU cores increases. This is because the proportion of the replay buffer operations time decreases for each core and the bottleneck shifts to training the neural networks.

\begin{figure}
    \centering
    \includegraphics[width=\linewidth]{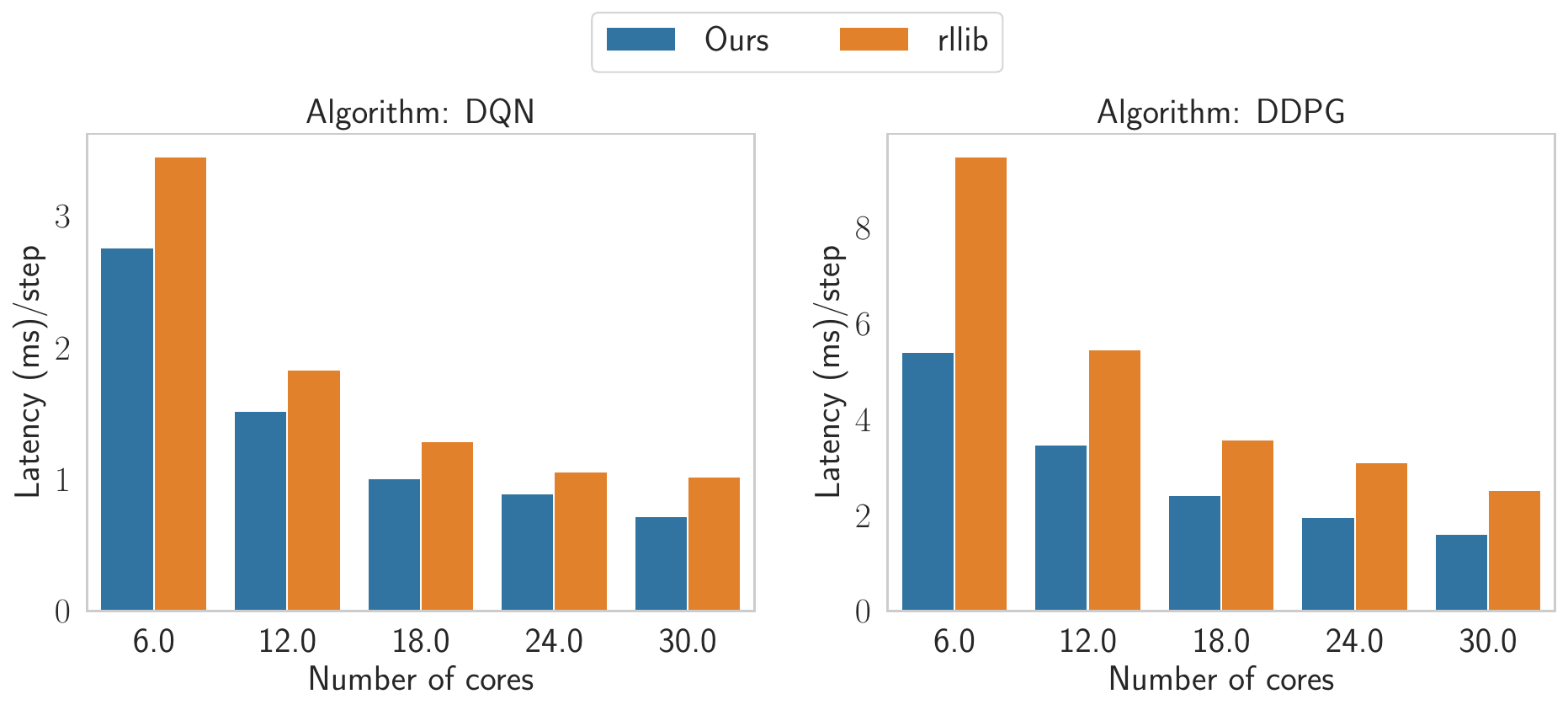}
    \caption{Overall speedup by plugging our prioritized replay buffer implementation into existing open source reinforcement learning libraries.}
    \label{fig:library_perf}
\end{figure}


\subsection{Design Space Exploration}

\begin{figure}
    \centering
    \includegraphics[width=\linewidth]{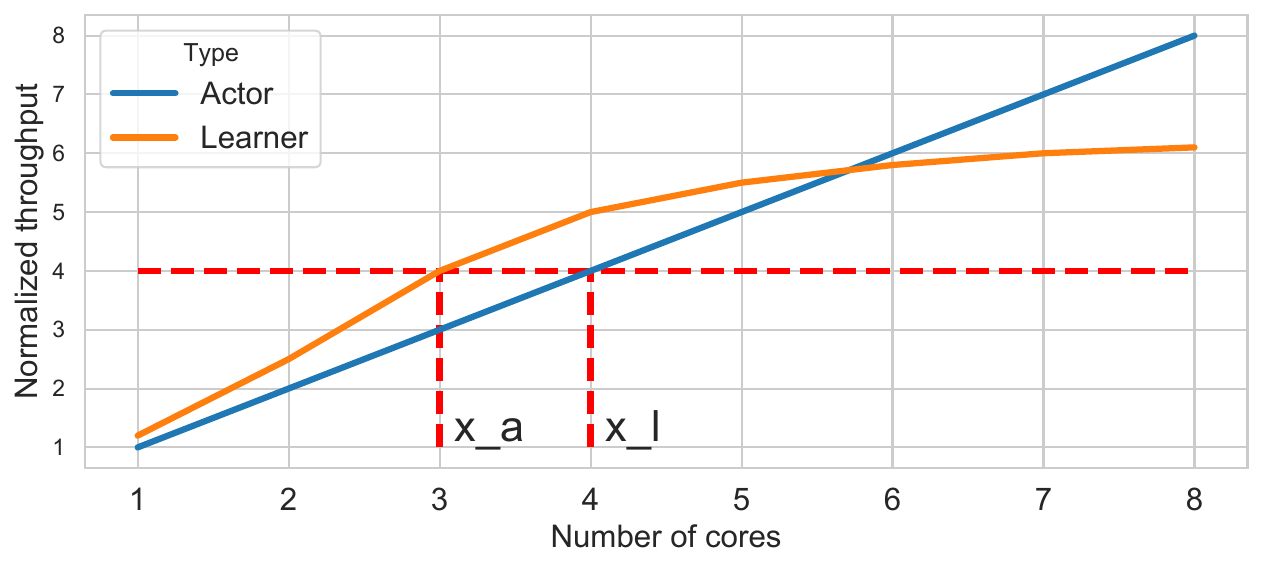}
    \caption{Illustration of design space exploration}
    \label{fig:design_space}
\end{figure}

As discussed in Section~\ref{sec:design_space_exp}, the objective is to allocate the number of cores for actors and learners, respectively such that the desired throughput ratio of the data collection and the data consumption is met. Our framework will first profile the throughput curve of actors and learners. We show an example in Figure~\ref{fig:design_space}, where the desired throughput ratio is 1. Then, we perform exhaustive search to find the solution $x_a$ and $x_b$ of Equation~\ref{eq:design_space}. The time complexity of the exhaustive search is $O(M^2)$, where $M$ is the total number of cores in the processor.

\subsection{Impact of the Data Layout}
The total size of the sum tree used in a typical replay buffer of size 1 million is less 10 KB. This makes the whole sum tree fit into the L2 cache of the modern CPUs. Thus, we only observe around $1\%$ benefit of our proposed cache aligned data layout. However, as the increase of the replay buffer size on larger problems, the superiority of our proposed data layout will appear.

\section{Conclusions and Future Work}
In this work, we propose a framework for generating scalable RL implementations on processor with accelerator platforms. We propose to use parallel actors and learners to increase the throughput of the data collection and the data consumption. To support asynchronous actors and learners, we propose a Prioritized Replay Buffer based on $K$-ary sum tree data structure. We propose \textit{lazy writing} locking mechanism to minimize the synchronization effort. Our experiments demonstrate that our proposed framework is superior to baseline approaches. Given hardware resources, our framework can automatically generate the number of actor threads and learner threads such that the desired ratio between data collection and data consumption is met. Future work includes implementation of the learners on various accelerator types including GPU clusters and FPGAs.

\bibliographystyle{IEEEtran}
\bibliography{bib/chi_bib}

\end{document}